\documentclass[letterpaper]{article} 
\usepackage[preprint]{aaai2027}  
\usepackage[hyphens]{url}  
\usepackage{graphicx} 
\usepackage{natbib}  
\usepackage{caption} 
\usepackage{algorithm}
\usepackage{amsmath}
\usepackage{amssymb}
 \usepackage{multirow} 
\usepackage{mathtools}
\usepackage{amsthm}
\usepackage{booktabs}
\usepackage{tabularx}
\usepackage{microtype}
\usepackage{graphicx}
\usepackage{subcaption}
\usepackage{booktabs} 
\usepackage{threeparttable}
\usepackage{amsmath,amssymb,amsthm,booktabs,algpseudocode}
\newtheorem{proposition}{Proposition}

\usepackage{newfloat}
\usepackage{listings}
\usepackage{amsmath,amssymb,amsthm}

\DeclareMathOperator{\clip}{clip}
\DeclareCaptionStyle{ruled}{labelfont=normalfont,labelsep=colon,strut=off} 
\floatstyle{ruled}
\newfloat{listing}{tb}{lst}{}
\floatname{listing}{Listing}

\usepackage{booktabs}

\title{Matterhorn: Masked Time-to-First-Spike Encoding by Reassigning the Silent State \\ for Sparse and Energy-Efficient Spiking Transformers }
\author{
    Zhanglu Yan\textsuperscript{\rm 1},
    Kaiwen Tang\textsuperscript{\rm 1},
    Zixuan Zhu\textsuperscript{\rm 1,\rm 3,\rm 4},
    Zhenyu Bai\textsuperscript{\rm 1},
    Qianhui Liu\textsuperscript{\rm 2},
    Yongxin Zhu\textsuperscript{\rm 3},
    Weng-Fai Wong\textsuperscript{\rm 1}
}
\affiliations{
    \textsuperscript{\rm 1}Department of Computer Science, National University of Singapore\\
    \textsuperscript{\rm 2}School of Artificial Intelligence, Shandong University\\
    \textsuperscript{\rm 3}Shanghai Advanced Research Institute, Chinese Academy of Sciences\\
    \textsuperscript{\rm 4}University of Chinese Academy of Sciences, Beijing
}

\begin{document}

\maketitle

\begin{abstract}
Spiking neural networks (SNNs) promise energy-efficient inference for large language models (LLMs), yet most reported savings rely on compute-operation counts that overlook data movement. Energy characterization of representative spiking transformers on a commercial 22-nm process shows that accumulation contributes less than 3\% of total energy, while spike-triggered inter-core transfers and weight reads dominate the cost. This makes time-to-first-spike (TTFS) encoding a natural choice, as it limits each neuron to at most one spike. However, standard TTFS maps the silent state—an all-zero spike train that transmits no events—to the rarely occurring smallest value, while the most common activations still spike. This raises a simple question: why reserve the only cost-free codeword for a rare value? This choice inverts a basic principle of energy-aware coding, under which the zero-event codeword should represent the most common value, rather than a rare extreme. Thus, we introduce masked time-to-first-spike encoding (M-TTFS), which uses a temporal mask to reassign the silent state to the most common activation value, and a dead-zone extension that trades a controlled amount of information for greater sparsity. 
Built on M-TTFS with dead-zone radius $k{=}1$, our spiking
transformer Matterhorn reaches an overall spike rate of 1.64\% on
GLUE at an average score of 84.64, exceeding the best prior
spiking transformer by 1.42 percentage points while consuming
67\% less energy, with consistent gains on spiking LLaMA models
from 7B to 70B parameters. Together, these results show that under hardware-faithful accounting, the energy advantage of SNNs is not a given: it is earned by encodings that align spikes with the data distribution.

\end{abstract}

\section{Introduction}

Spiking neural networks (SNNs) are increasingly studied for energy-efficient
large language model (LLM) inference~\cite{wang2025energy, xing2024spikellm}. The promise rests largely on a single substitution: replacing multiply-accumulate (MAC) operations with cheaper accumulate (ACC) operations. On this basis, Sorbet~\cite{tangsorbet} and SpikeLM~\cite{xingspikelm} report 27.16\texttimes{} and 12.9\texttimes{} efficiency gains, estimated from operation counts alone comparing with baseline quantized ANNs (QNNs). However, such estimates miss most of the real hardware cost: our energy characterization of state-of-the-art spiking transformers on a commercial 22nm process (Figure~\ref{fig:energy_breakdown}) shows that ACC computation accounts for only 1.4-2.6\% of total energy but data movement dominates the rest. 

\begin{figure}[t]
    \centering
    \includegraphics[width=0.91\linewidth]{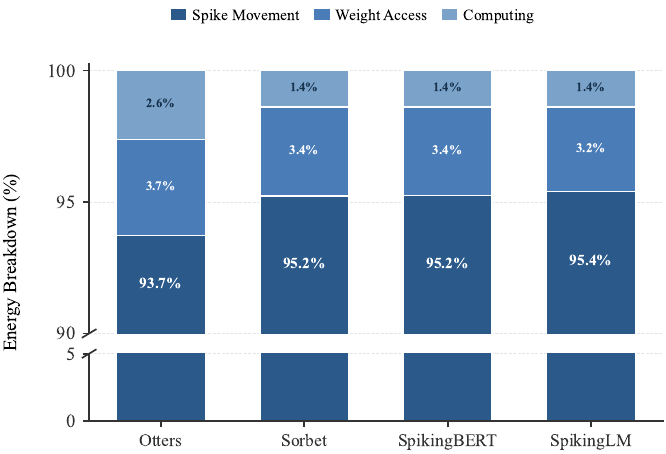}
    \caption{Proportional Energy Breakdown of SOTA spiking transformers. Otters (we use its standard-TTFS variant), Sorbet (rate encoding), SpikingBERT (rate encoding), and SpikingLM (rate encoding) are evaluated with time steps $T$ set to 15, 16, 16, and 4, and average spike rates of 5.14\%, 13\%, 25\%, and 33\%, respectively, as reported in~\cite{bal2024spikingbert, xingspikelm, tangsorbet, yan2025otters}. SNN models are assumed to be deployed on a spatial dataflow architecture as described in Section~\ref{sec: energy_cal}.}
    \label{fig:energy_breakdown}
\end{figure}

Yet SNNs have no inherent advantage in this dominant cost. A QNN transmits
each activation once as a dense $n$-bit integer, whereas an SNN represents the
same activation as events over a $T$-step window. This event-based
representation is a double-edged sword: sparsity can reduce the number of
transfers, but each sparse event carries address and routing overhead, giving
it a higher effective transfer cost than dense QNN streaming
~\cite{davies2018loihi,bai2025data}. Each firing also triggers an associated
weight read, another cost omitted by operation-count estimates. Thus, sparse
spikes are not automatically cheaper; they must be sufficiently rare to offset
their higher per-event cost. A recent end-to-end comparison between SNNs and
their quantized twins estimates that, under representative neuromorphic
settings, an SNN becomes more energy-efficient only when the expected number
of spikes per neuron per window, denoted by $s$, falls below approximately
$0.35$. For $T=16$, this corresponds to a per-timestep spike rate below
$2.2\%$~\cite{yan2024reconsidering}.

Within an SNN, the encoding determines the spike count. Rate encoding, the prevailing choice in spiking transformers, represents magnitude by firing frequency, so spikes grow with activation value: reported spike rates of 13--33\% with around 16 time steps~\cite{bal2024spikingbert,xingspikelm} sit roughly an order of magnitude above the crossover. 
Time-to-first-spike (TTFS) limits each neuron to at most one spike~\cite{sakemi2023sparse,sun2025temporal,lew2022ttfstraining}, yet standard TTFS still produces a spike rate of 3.74\% at $T=16$ ($s\approx0.60$), well above the
crossover. The problem lies in its code assignment: standard TTFS reserves the
silent all-zero codeword for the smallest representable value, a rare negative
extreme under symmetric quantization. This raises a simple but overlooked question: \textbf{why reserve the only cost-free codeword for a rare value?} The basic principle of energy-aware coding says the opposite: the cheapest codeword belongs to the most frequent symbol; after all, the cheapest event is not a one-bit spike, but no spike at all.

We therefore propose \emph{masked time-to-first-spike} (M-TTFS),
a distribution-aware extension of TTFS that reassigns the silent
codeword to the most frequent activation value. Instead of allowing
silence to represent a rare boundary value, M-TTFS masks the dominant
spike time and uses the freed terminal time to encode the original
lowest value. We show that this assignment is optimal among all
single-silence, single-spike encodings by minimizing the expected
spike count. Moreover, with a simple offset folding into the next
layer's bias, M-TTFS preserves the source QNN computation exactly:
silence changes the representation, not the information. Beyond this
lossless setting, we introduce a dead-zone extension that expands the
silent region around the mode with a controllable radius \(k\),
trading a bounded accuracy loss for additional sparsity. 
\textbf{TTFS limits how often an activation can spike; M-TTFS decides
which activation values deserve to remain silent.}

Built on M-TTFS, we construct \emph{Matterhorn}, a spiking
transformer obtained by QNN-to-SNN
conversion.
At the standard $T=16$ setting, the lossless reassignment alone
($k=0$) cuts the spike rate from 3.74\% to 2.60\% and the energy
per encoder block on SST-2 from 117.5 to 77.19\,mJ; at $k=1$, the
rate falls to 1.64\% ($s\approx0.26$) and the energy to
48.27\,mJ, which is 33\% below the 71.54\,mJ of the dense INT4
network it is converted from, an advantage that holds on every individual
task and across the entire plausible per-event cost range. On GLUE, Matterhorn ($k=1$)
surpasses the best prior spiking transformer by 1.42 percentage
points in average score with 67\% less energy. Applied to spiking LLaMA models from 7B to 70B parameters,
lossless M-TTFS reduces spike rate at identical perplexity,
while larger dead-zone radii provide further reductions with
a controlled perplexity increase.
Taken together, these results make a simple point: the energy
efficiency of SNNs does not come from spiking itself; it must be
earned by an encoding that places silence on the mode of the data
distribution.

\section{Preliminary}

\subsection{Time-To-First-Spike encoding}\label{sec: ttfs}

 {\em Time-to-first-spike} (TTFS) encoding carries information solely in the arrival time of the first and only spike within the window $T$~\cite{zhao2025ttfsformer}. The operation of a standard TTFS neuron proceeds in two phases: integration and threshold comparison. First, the neuron decodes temporal information using a kernel decay function $f(t)$ (typically $f(t) = T - t$)~\cite{yan2025otters}. It updates its membrane potential $V_j^l(t)$ by accumulating weighted inputs from presynaptic neurons:
\begin{equation}
    V_j^l(t) = V_j^l(t-1) + \sum_{i} w^l_{ij} \cdot s^{l-1}_i(t) \cdot f(t)
    \label{eq: membrane_potential_acc}
\end{equation}

Second, the neuron compares this potential to a firing threshold $\theta^l(t)$. A spike is generated at the specific time step $t$ where the potential meets or exceeds the threshold.

Thus, the firing time is the first instance $t$ such that $V_j^l(t) \ge \theta^l(t)$. If the accumulated potential never exceeds the threshold, the neuron remains silent (a spike train of all zeros). While this silent state maximizes energy efficiency by eliminating data movement costs, it introduces a trade-off. Standard TTFS often suffers from ``over-sparsity,'' where neurons that fail to fire cannot effectively learn or propagate information~\cite{wei2023temporal}. To address this, we adopt the Dynamic Firing Threshold (DFT) model~\cite{wei2023temporal}, which enforces a synchronous, layer-by-layer schedule to preserve causal relationships. Furthermore, to avoid the convergence instability of direct training, we employ a QNN-to-SNN conversion framework~\cite{yan2025otters}.

\subsection{Energy calculation}\label{sec: energy_cal}

 The total energy consumption is decomposed into two primary components: Computation and Data Movement. Computation includes membrane potential accumulation, decay operations, and threshold comparisons. Data movement encompasses inter-core spike transfer, SRAM access for weights/thresholds, and KV cache read/write~\cite{yan2024reconsidering}. Our analysis assumes a spatial dataflow architecture where information (e.g., spike packets) is communicated over a Network-on-Chip (NoC). This design is representative of modern specialized hardware, including neuromorphic chips like Loihi~\cite{davies2018loihi} and dataflow accelerators such as Sambanova~\cite{prabhakar2022sambanova}. The energy for the Fully Connected (FC) layer ($E_{\text{FC}}$) and Attention Q/K/V computation ($E_{\text{qkv}}$) is defined as:

\begin{equation}
\begin{split}
    E_{\text{FC}} &= N^{\text{fc}}_o \cdot \Bigg[ 
    C_i T  s_r (\underbrace{E^{\text{ACC}}_{\text{Decay}} + E_{\text{MAC}} + E^{\text{Read}}_{\text{W}} + E^{\text{sparse}}_{\text{move}}}_{\text{Spike Processing}}) \\
    &   + \underbrace{T (E_{\text{CMP}} + E^{\text{Read}}_{\text{th}})}_{\text{Thresholding}} + \underbrace{E^{\text{Write}}_{\text{kv}}}_{\text{K/V Write}} \Bigg]
\end{split}
\label{eq: opt-fc}
\end{equation}

\begin{equation}
\begin{split}
E_{\text{qkv}}
={}&
N^{\text{qkv}}_o
\Bigg[
d_k T s_r
\underbrace{
\left(
E^{\text{ACC}}_{\text{Enc}}
+E_{\text{MAC}}
+E^{\text{sparse}}_{\text{move}}
+E^{\text{Read}}_{\text{KV}}
\right)
}_{\text{Spike Processing}}
\\
&\qquad+
\underbrace{
T\left(
E_{\text{CMP}}+E^{\text{Read}}_{\text{th}}
\right)
}_{\text{Thresholding}}
\Bigg].
\end{split}
\label{eq:opto-score}
\end{equation}

where $N^{\text{fc}}_o = B \cdot S \cdot C_o$ and $N^{\text{qkv}}_o = B \cdot h \cdot S^2$ represent the total number of output elements. Here, $B$ denotes the batch size, $S$ is the sequence length, $C_i/C_o$ are the input/output channel dimensions, and $d_k$ is the per-head dimension. The term $s_r$ denotes the input spike ratio, which linearly scales the spike processing energy.

\section{Method}
\label{sec:method}

M-TTFS reassigns the zero-cost silent state from the lowest
quantized activation to the most frequent activation region.
Section~\ref{sec:mttfs} defines the event-time assignment and
distribution-aware temporal mask. Section~\ref{sec:offset_fold}
preserves the downstream weighted sum when the selected silent
value is nonzero. Section~\ref{sec:qnn_conversion} identifies the
dead-zone-projected QNN exactly implemented by M-TTFS, and
Section~\ref{sec:dz_training} presents its training procedure.

\subsection{M-TTFS: Distribution-Aware Silent-State Assignment}
\label{sec:mttfs}

An \(n\)-bit quantized activation has \(T=2^n\) representable
integer levels. Let \(c\) denote the largest level, so the complete
set is
\[
    \{c-(T-1),\ldots,c\}.
\]
Standard TTFS represents the lowest level by silence and the
remaining \(T-1\) levels by events. M-TTFS instead represents the lowest level by a terminal
event at \(t=T-1\), freeing silence for a distribution-aware
assignment. Similar decreasing threshold schedule is used in M-TTFS:
\begin{equation}
    \theta^\ell(t)
    =
    \alpha^\ell(c-t),
    \qquad
    t\in\{0,\ldots,T-2\},
\label{eq:threshold_schedule}
\end{equation}
where \(\alpha^\ell>0\) is the layer-wise activation scale.

Defining the set of ordinary event times as
\begin{equation}
    \mathcal{T}_j^\ell
    =
    \left\{
    t\in\{0,\ldots,T-2\}:
    V_j^\ell\geq\theta^\ell(t)
    \right\}.
\label{eq:ordinary_event_times}
\end{equation}
We assign neuron \(j\) the event time
\begin{equation}
    \tau_j^\ell
    =
    \min
    \left(
    \mathcal{T}_j^\ell
    \cup
    \{T-1\}
    \right).
\label{eq:event_time}
\end{equation}

Thus, \(\tau_j^\ell\) is the earliest time step at which the
membrane potential meets the threshold. When no ordinary time step
meets the threshold, M-TTFS assigns the terminal time \(T-1\).
Every activation therefore has exactly one event time before
masking.

We estimate the pre-mask event-time distribution of each layer on a held-out calibration set \(\mathcal C\). For a layer with \(N^\ell\)
activation elements per input, we define
\begin{equation}
    \widehat{p}_t^\ell
    =
    \frac{1}{|\mathcal C|N^\ell}
    \sum_{x\in\mathcal C}
    \sum_{j=1}^{N^\ell}
    \mathbb{I}
    \left(
    \tau_{j,x}^\ell=t
    \right),
    \qquad
    t\in\{0,\ldots,T-1\}.
\label{eq:mttfs_histogram}
\end{equation}
Section~\ref{sec:qnn_conversion} establishes a bijection between
quantized integer levels and event times. Therefore,
Eq.~\eqref{eq:mttfs_histogram} can be computed directly from QNN
activation histograms, without time-step SNN simulation.

Then, based on $ \widehat{p}_t^\ell$, the most frequent event time can be selected as:
\begin{equation}
    I_{\max}^\ell
    \in
    \arg\max_{t\in\{0,\ldots,T-1\}}
    \widehat{p}_t^\ell.
\label{eq:mttfs_imax}
\end{equation}
We then optionally expand the silent set to a radius-\(k\)
neighbourhood around it:
\begin{equation}
    \Omega_k^\ell
    =
    \left\{
    t\in\{0,\ldots,T-1\}:
    \left|t-I_{\max}^\ell\right|\leq k
    \right\},
    \qquad
    k\geq0.
\label{eq:dz_set}
\end{equation}
The transmitted spike train is
\begin{equation}
    s_j^\ell(t)
    =
    \mathbb{I}
    \left(
    t=\tau_j^\ell
    \right)
    \mathbb{I}
    \left(
    t\notin\Omega_k^\ell
    \right),
    \qquad
    t\in\{0,\ldots,T-1\}.
\label{eq:mttfs_mask}
\end{equation}

\begin{figure}[htb]
    \centering
    \includegraphics[width=0.93\linewidth]{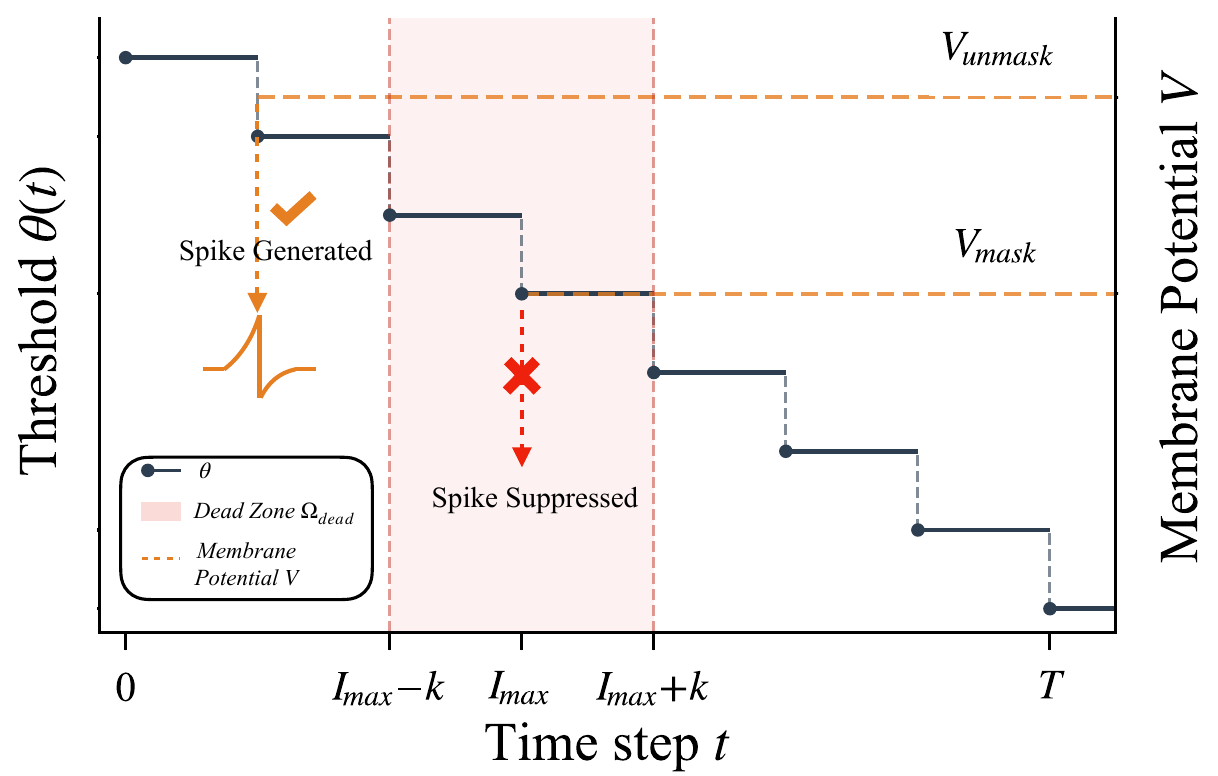}
    \caption{
    M-TTFS assigns silence to the most frequent event time.
    Increasing \(k\) expands the silent set to a temporal dead zone
    around \(I_{\max}^\ell\).
    }
    \label{fig:mask}
\end{figure}

Core M-TTFS uses \(k=0\), assigning silence to the most frequent
event time. The dead-zone extension uses \(k\geq1\) and suppresses
a wider neighbourhood around it. Because \(\tau_j^\ell\) is fixed
before masking, suppression removes the event rather than moving it
to a later time step.

Let
\(\widehat{s}_{\mathrm{tx}}^\ell(\Omega_k^\ell)\)
denote the average number of transmitted events per activation on
the calibration set. Since every activation has exactly one event
time before masking,
\begin{equation}
    \widehat{s}_{\mathrm{tx}}^\ell(\Omega_k^\ell)
    =
    1-
    \sum_{t\in\Omega_k^\ell}
    \widehat{p}_t^\ell.
\label{eq:dz_event_count}
\end{equation}
The transmitted activity therefore cannot increase as the silent
set grows. In particular, core M-TTFS transmits
\(
1-\widehat{p}_{I_{\max}^\ell}^\ell
\)
events per activation.

\begin{proposition}[Optimal single-time silence assignment]
\label{prop:optimal_mask}
Among all assignments that map one event time to silence, core
M-TTFS minimizes the average number of transmitted events on the
calibration set.
\end{proposition}

\begin{proof}
Assigning silence to event time \(r\) suppresses a fraction
\(\widehat{p}_r^\ell\) of activations and yields
\[
    \widehat{s}_{\mathrm{tx}}^\ell(r)
    =
    1-\widehat{p}_r^\ell.
\]
This quantity is minimized by
\(r\in\arg\max_t\widehat{p}_t^\ell\), which is precisely
\(I_{\max}^\ell\) in Eq.~\eqref{eq:mttfs_imax}.
\end{proof}

\subsection{Offset Folding for Nonzero Silent Values}
\label{sec:offset_fold}

Under event-driven accumulation, an event at time \(t\) contributes
through the linear decay function
\begin{equation}
    f^\ell(t)
    =
    \alpha^\ell(c-t),
    \qquad
    t\in\{0,\ldots,T-1\}.
\label{eq:decay_function}
\end{equation}
An event at time \(t\) therefore represents \(f^\ell(t)\). In
particular,
\[
    f^\ell(T-1)
    =
    \alpha^\ell\bigl(c-(T-1)\bigr),
\]
so the terminal event represents the lowest quantized activation.
Silence contributes zero under event-driven accumulation. Therefore,
directly suppressing an event that represents a nonzero value would
change the downstream weighted sum. To give the silent region a
well-defined numerical value, we use the value represented by its
centre time. Specifically,
\begin{equation}
    \mu^\ell
    =
    c-I_{\max}^\ell
\label{eq:silent_integer}
\end{equation}
is the quantized integer associated with
\(I_{\max}^\ell\), and its activation value is
\(\alpha^\ell\mu^\ell\). M-TTFS assigns this value to all event
times in \(\Omega_k^\ell\).

To realize this assignment under event-driven accumulation, each
transmitted event carries its deviation from
\(\alpha^\ell\mu^\ell\). We therefore define the centred decay
function
\begin{equation}
    \widetilde{f}^\ell(t)
    =
    f^\ell(t)
    -
    \alpha^\ell\mu^\ell
    =
    \alpha^\ell
    \left(
    c-t-\mu^\ell
    \right).
\label{eq:centered_decay}
\end{equation}
The removed constant is restored by folding it into the bias of the
next layer:
\begin{equation}
    \widetilde{b}_j^{\ell+1}
    =
    b_j^{\ell+1}
    +
    \alpha^\ell\mu^\ell
    \sum_i
    w_{ij}^{\ell+1}.
\label{eq:folded_bias}
\end{equation}
After layer \(\ell\) completes its output window, layer \(\ell+1\)
integrates
\begin{equation}
    V_j^{\ell+1}
    =
    \widetilde{b}_j^{\ell+1}
    +
    \sum_i
    w_{ij}^{\ell+1}
    \sum_{t=0}^{T-1}
    s_i^\ell(t)
    \widetilde{f}^\ell(t).
\label{eq:mttfs_membrane}
\end{equation}
The folded bias restores the common contribution
\(\alpha^\ell\mu^\ell\), while each transmitted event contributes
only its deviation from this value.
Section~\ref{sec:qnn_conversion} then formalizes this assignment as a
dead-zone projection and proves the resulting QNN--M-TTFS
equivalence.

\subsection{QNN-to-M-TTFS Conversion}
\label{sec:qnn_conversion}

We now define the QNN matched by the M-TTFS construction above and
prove their layer-wise equivalence. Let
\(\bar{x}_{q,i}^{\ell-1}
=\alpha^{\ell-1}\bar{q}_i^{\ell-1}\)
denote the dead-zone-adjusted activation from the preceding layer.
Layer \(\ell\) computes
\begin{align}
    a_j^\ell
    &=
    b_j^\ell
    +
    \sum_i
    w_{ij}^\ell
    \bar{x}_{q,i}^{\ell-1},
\label{eq:qnn_pre_activation}\\
    q_j^\ell
    &=
    \clip
    \left(
    \left\lfloor
    \frac{a_j^\ell}{\alpha^\ell}
    \right\rfloor,
    c-(T-1),
    c
    \right),
\label{eq:qnn_code}
\end{align}
where \(a_j^\ell\) is the pre-activation and \(q_j^\ell\) is the
quantized integer before applying the dead zone.

The QNN applies the same dead zone as M-TTFS by mapping integer
values around \(\mu^\ell\) to \(\mu^\ell\):
\begin{equation}
    \bar{q}_j^\ell
    =
    \begin{cases}
        \mu^\ell,
        &
        \left|q_j^\ell-\mu^\ell\right|\leq k,
        \\[3pt]
        q_j^\ell,
        &
        \text{otherwise}.
    \end{cases}
\label{eq:qnn_deadzone}
\end{equation}
The activation passed to the next layer is
\begin{equation}
    \bar{x}_{q,j}^\ell
    =
    \alpha^\ell\bar{q}_j^\ell.
\label{eq:qnn_activation}
\end{equation}
When \(k=0\), Eq.~\eqref{eq:qnn_deadzone} leaves every quantized
integer unchanged. When \(k\geq1\), it maps a neighbourhood around
\(\mu^\ell\) to the value represented by silence.

Since $q_j^\ell=\clip(\lfloor a_j^\ell/\alpha^\ell\rfloor,
c-(T-1), c)$ implies $a_j^\ell \ge \theta^\ell(t)$ exactly when
$t \ge c - a_j^\ell/\alpha^\ell$, the earliest ordinary time in
Eq.~\eqref{eq:ordinary_event_times} is $c-q_j^\ell$ whenever
$q_j^\ell > c-(T-1)$, and the lowest code, which meets no
ordinary threshold, receives the terminal time
$T-1 = c-\bigl(c-(T-1)\bigr)$. Hence
\begin{equation}
    \tau(q)=c-q.
\label{eq:q_to_time}
\end{equation}
This mapping establishes a bijection between the \(T\) integer levels
and the \(T\) event times.
Using
\(
\mu^\ell=c-I_{\max}^\ell
\),
we obtain
\begin{equation}
    \left|
    \tau_j^\ell-I_{\max}^\ell
    \right|
    =
    \left|
    q_j^\ell-\mu^\ell
    \right|.
\label{eq:q_time_distance}
\end{equation}
Consequently,
\begin{equation}
    \tau_j^\ell\in\Omega_k^\ell
    \quad\Longleftrightarrow\quad
    \left|q_j^\ell-\mu^\ell\right|\leq k.
\label{eq:q_time_deadzone}
\end{equation}
Therefore, M-TTFS suppresses an event exactly when the matching QNN
maps its quantized integer to \(\mu^\ell\).

\begin{proposition}[M-TTFS--QNN equivalence]
\label{prop:mttfs_equivalence}
Under the layer-sequential accumulate-then-fire schedule, and
assuming that network inputs use the same quantization and event-time
encoding, M-TTFS with radius \(k\geq0\) implements the same
layer-wise computation as the QNN defined by
Eqs.~\eqref{eq:qnn_pre_activation}--\eqref{eq:qnn_activation}.
\end{proposition}

\begin{proof}
The shared input quantization and event-time encoding establish the
base case. Suppose that layer \(\ell\) represents the QNN output
\(
\bar{x}_{q,i}^\ell
=
\alpha^\ell\bar{q}_i^\ell
\).

If
\(
|q_i^\ell-\mu^\ell|\leq k
\),
Eq.~\eqref{eq:q_time_deadzone} places the event inside
\(\Omega_k^\ell\). The event is suppressed, and
\[
    \sum_{t=0}^{T-1}
    s_i^\ell(t)\widetilde{f}^\ell(t)
    =
    0
    =
    \alpha^\ell
    \left(
    \bar{q}_i^\ell-\mu^\ell
    \right).
\]
Otherwise, one event is transmitted at
\(\tau_i^\ell=c-q_i^\ell\), giving
\[
\begin{aligned}
    \sum_{t=0}^{T-1}
    s_i^\ell(t)\widetilde{f}^\ell(t)
    &=
    \widetilde{f}^\ell(\tau_i^\ell)=
    \alpha^\ell
    \left(
    q_i^\ell-\mu^\ell
    \right) =
    \alpha^\ell
    \left(
    \bar{q}_i^\ell-\mu^\ell
    \right).
\end{aligned}
\]
Substituting this relation and the folded bias from
Eq.~\eqref{eq:folded_bias} into the next layer gives
\begin{align}
    V_j^{\ell+1}
    &=
    \widetilde{b}_j^{\ell+1}
    +
    \sum_i
    w_{ij}^{\ell+1}
    \alpha^\ell
    \left(
    \bar{q}_i^\ell-\mu^\ell
    \right)
    \nonumber\\
    &=
    b_j^{\ell+1}
    +
    \alpha^\ell\mu^\ell
    \sum_iw_{ij}^{\ell+1}
    \nonumber +
    \sum_i
    w_{ij}^{\ell+1}
    \alpha^\ell
    \left(
    \bar{q}_i^\ell-\mu^\ell
    \right)
    \nonumber\\
    &=
    b_j^{\ell+1}
    +
    \sum_i
    w_{ij}^{\ell+1}
    \bar{x}_{q,i}^\ell
    \nonumber\\
    &=
    a_j^{\ell+1}.
\label{eq:proof_membrane_equivalence}
\end{align}
The threshold schedule then assigns
\(q_j^{\ell+1}\) the event time
\(
\tau_j^{\ell+1}=c-q_j^{\ell+1}
\),
and Eq.~\eqref{eq:q_time_deadzone} applies exactly the QNN dead zone
in Eq.~\eqref{eq:qnn_deadzone}. Hence, the M-TTFS output represents
\(\bar{x}_{q,j}^{\ell+1}\), completing the induction. 
\end{proof}

When \(k=0\), Eq.~\eqref{eq:qnn_deadzone} is the identity, so core
M-TTFS is exactly equivalent to the original QNN and requires no
additional training. For \(k\geq1\), M-TTFS is exactly equivalent
to the dead-zone QNN defined above, which we train in
Section~\ref{sec:dz_training}.

\subsection{Training the Dead-Zone QNN}
\label{sec:dz_training}

Proposition~\ref{prop:mttfs_equivalence} shows that M-TTFS-DZ
implements the same computation as the corresponding dead-zone QNN.
We therefore perform all optimization in the QNN domain and convert
the trained model to M-TTFS afterward. Core M-TTFS (\(k=0\))
requires no additional training because
Eq.~\eqref{eq:qnn_deadzone} leaves every quantized integer
unchanged. For \(k\geq1\), we apply the dead-zone mapping after each
activation quantizer and fine-tune the resulting QNN.

For implementation, we define the projection mask
\begin{equation}
    M_{\mathrm{qnn},j}^\ell
    =
    \mathbb{I}
    \left(
    |q_j^\ell-\mu^\ell|>k
    \right).
\label{eq:qnn_mask}
\end{equation}
The projected integer can then be written as
\begin{equation}
    \bar{q}_j^\ell
    =
    M_{\mathrm{qnn},j}^\ell q_j^\ell
    +
    \left(
    1-M_{\mathrm{qnn},j}^\ell
    \right)\mu^\ell.
\label{eq:qnn_deadzone_mask}
\end{equation}

The quantizer and projection are non-differentiable. We use a masked
straight-through estimator:
\begin{equation}
\begin{split}
    \frac{\partial\mathcal L}
    {\partial a_j^\ell}
    \approx\;&
    \frac{\partial\mathcal L}
    {\partial\bar{x}_{q,j}^\ell}
    \,
    \mathbb{I}
    \left(
    q_{\min}
    \leq
    \frac{a_j^\ell}{\alpha^\ell}
    <
    q_{\max}+1
    \right)
    M_{\mathrm{qnn},j}^\ell.
\end{split}
\label{eq:masked_ste}
\end{equation}
For compactness, we write
$q_{\min}=c-(T-1)$ and $q_{\max}=c$.
The first indicator passes gradients through the unsaturated range
of the quantizer. The second blocks gradients for values collapsed
by the dead-zone projection, whose local output is constant with
respect to \(a_j^\ell\).

We fine-tune the projected QNN using knowledge distillation
\cite{liu2022bit}. The student minimizes
\begin{equation}
    \mathcal L
    =
    \mathcal L_{\mathrm{task}}
    +
    \lambda_{\mathrm{KD}}
    \mathcal L_{\mathrm{KD}},
\label{eq:training_loss}
\end{equation}
where \(\mathcal L_{\mathrm{task}}\) is the supervised task loss and
\(\mathcal L_{\mathrm{KD}}\) matches the student outputs to those of
its full-precision teacher. The teacher, loss definition,
\(\lambda_{\mathrm{KD}}\), and optimization schedule are specified
in the experimental setup.

The silent centres \(\{\mu^\ell\}\) and dead-zone radius \(k\) are
fixed during fine-tuning. We calibrate the centres using the
held-out set \(\mathcal C\) and select \(k\) according to the
validation quality--energy trade-off, without using test-set
statistics. After fine-tuning, the QNN is converted to M-TTFS using
the event-time mapping and fixed temporal masks defined in
Sections~\ref{sec:mttfs} and~\ref{sec:qnn_conversion}. No online
histogram estimation or adaptive centre selection is required
during inference.

\section{Results}

In this section, we evaluate on seven GLUE development sets (accuracy; Pearson correlation for STS-B; MNLI matched). Matterhorn is obtained by
converting a W1A4 QNN (1-bit weights, 4-bit activations, hence
$T=2^{4}=16$ time steps) distilled from a full-precision
BERT$_{\text{base}}$ teacher.
Eight quantized tensor
positions per encoder block are spike-encoded: the three attention
projection inputs, the query activation, the attention
probabilities, the attention-output input, and the two
feed-forward inputs. Key and value activations are not converted to spikes and are treated as weights~\cite{tangsorbet}.
To quantify energy consumption, we synthesized verilog implementations of key components on a commercial 22~nm process, measuring unit costs for operations, and per-bit energy for inter-core spike movement and SRAM access.


\subsection{GLUE Benchmark Performance}

As detailed in Table~\ref{tbl:glue_performance_avg}, our proposed M-TTFS-based Matterhorn establishes a state-of-the-art accuracy among spiking transformers, both with and without the dead zone constraint ($k=1$ and $k=0$). When compared to existing 1-bit baselines with comparable parameter counts, Matterhorn consistently outperforms the previous leading method, Spiking Otters~\cite{yan2025otters}. Specifically, with a dead zone radius of $k=0$, it achieves an average score of 85.87 (a 2.65-percentage-point improvement), and maintains a 1.42-point lead even under the stricter sparsity constraint of $k=1$ (84.64). On the challenging RTE task, Matterhorn ($k=0$) reaches 72.56\%, exceeding Spiking Otters by over 3.6 percentage points and approaching the performance of the full-precision BERT baseline. Furthermore, despite its reduced size, Matterhorn ($k=0$) surpasses larger spiking baselines, outperforming the 50M-parameter SpikingBERT (80.83\%) by 5.04 points. 

The $k{=}0$ row also reports the accuracy of the W1A4 source
QNN. Standard TTFS and M-TTFS ($k{=}0$) are bit-exact
conversions of this model
(Proposition~\ref{prop:mttfs_equivalence}), so all three share
the same 85.87 average. The reassignment contributes lower event
traffic, not higher accuracy, and the energy comparison in
Section~\ref{sec:rate_energy} is therefore at equal accuracy.
\begin{table*}[hbt]
\small
\centering
\caption{Performance comparison on the GLUE benchmark. All scores represent accuracy except for STS B which uses Pearson correlation. The asterisk symbol indicates that the model size was not reported in the original paper. Bold text highlights the best performance among SNN models, while an \underbar{\text{underbar}} denotes the second-best result.}
\label{tbl:glue_performance_avg}
\scalebox{0.8}{
\begin{tabular}{@{}lcccccccccc@{}}
\toprule
\textbf{Model} & \textbf{Size} & \textbf{QQP} & \textbf{MNLI-m} & \textbf{SST-2} & \textbf{QNLI} & \textbf{RTE} & \textbf{MRPC} & \textbf{STS-B} & \textbf{Average} \\
\midrule
$\text{BERT}_{\text{base}}$~\cite{devlin2019bert} & 418M & 91.3 & 84.7 & 93.3 & 91.7 & 72.6 & 88.2 & 89.4 & 87.31 \\
DistilBERT~\cite{sanh2019distilbert} & 207M & 88.5 & 82.2 & 91.3 & 89.2 & 59.9 & 87.5 & 86.9 & 83.64 \\
$\text{TinyBERT}_6$~\cite{jiao2020tinybert} & 207M & - & 84.6 & 93.1 & 90.4 & 70.0 & 87.3 & 83.7 & 84.85 \\
BiT~\cite{liu2022bit} & 13.4M & 82.9 & 77.1 & 87.7 & 85.7 & 58.8 & 79.7 & 71.1 & 77.57 \\
SpikingFormer~\cite{zhou2025spikingformerkeyfoundationmodel} & * & 84.7 & 71.9 & 87.2 & 76.0 & 55.6 & 79.7 & 54.5 & 72.80 \\
SpikingBERT~\cite{bal2024spikingbert} & 50M & 86.8 & 78.1 & 88.2 & 85.2 & 66.1 & 79.2 & 82.2 & 80.83 \\
SpikeLM~\cite{xingspikelm} & * & 87.9 & 76.0 & 86.5 & 84.9 & 65.3 & 78.7 & 84.3 & 80.51 \\
\midrule
1-bit SpikeLM~\cite{xingspikelm} & * & 87.2 & 74.9 & 86.6 & 84.5 & 65.7 & 78.9 & 83.9 & 80.24 \\
1-bit Spiking Sorbet~\cite{tangsorbet} & 13.4M & 86.5 & 77.3 & 90.4 & 86.1 & 60.3 & 79.9 & 78.1 & 79.80 \\

\text{1-bit Spiking Otters}~\cite{yan2025otters} & 13.4M & \text{87.67} & \text{78.50} & \text{91.28} & \text{86.42} & \text{68.95} & \text{84.56} & \text{85.19} & \text{83.22} \\
\textbf{1-bit Matterhorn(k=0)} & 13.4M & \textbf{89.55} & \textbf{81.81} & \textbf{92.55} & \textbf{89.55} & \textbf{72.56} & \textbf{88.24} & \textbf{86.82} & \textbf{85.87} \\
\textbf{1-bit Matterhorn(k=1)} & 13.4M & \underbar{88.32} & \underbar{80.70} & \underbar{91.63} & \underbar{87.74} & \underbar{71.84} & \underbar{86.27} & \underbar{86.00} & \underbar{84.64} \\
\bottomrule
\end{tabular}}
\end{table*}

\subsection{Spike Rate and Energy}
\label{sec:rate_energy}

\begin{table*}[t]
\centering
\caption{Per-step spike rate (\%) over spike-encoded tensors on
the seven GLUE development sets ($T{=}16$). $s$ is spikes per
activation per window ($s=\text{rate}\times T$). Averages are computed from the unrounded task-level traces.
}
\label{tab:spike_rate}
\small
\scalebox{0.8}{
\begin{tabular}{@{}lccccccccc@{}}
\toprule
Encoding & QQP & MNLI-m & SST-2 & QNLI & RTE & MRPC & STS-B & Avg. Spike Rate & $s$ \\
\midrule
TTFS             & 3.77 & 3.75 & 3.85 & 3.70 & 3.57 & 3.68 & 3.85 & 3.74 & 0.60 \\
M-TTFS ($k{=}0$) & 2.34 & 2.82 & 2.52 & 2.72 & 2.60 & 2.57 & 2.62 & 2.60 & 0.42 \\
M-TTFS ($k{=}1$) & 1.36 & 1.93 & 1.57 & 1.62 & 1.60 & 1.64 & 1.80 & 1.64 & 0.26 \\
\bottomrule
\end{tabular}}
\end{table*}

\begin{table}[t]
\centering
\caption{Energy per transformer block on SST-2 (batch of 64,
$S{=}128$). Baselines are re-evaluated under identical unit costs
and mapping assumptions using their reported spike rates and time
steps. Rel.\ normalizes to Matterhorn's dense INT4 source QNN,
which streams every quantized activation as a 4-bit integer
without event overhead. Bit-serial QNN rows skip activations mapped to zero by the
corresponding dead-zone projection and process only active
binary bits of the remaining values. M-TTFS rows
transmit one payload-free event per spike (address overhead
included), with the magnitude encoded in its timing. Component
breakdowns and unit costs in Appendix~C.}
\label{tab:energy}
\small
\scalebox{0.7}{
\begin{tabular}{@{}llcccc@{}}
\toprule
Model & Encoding & $T$ & \shortstack{Rate\\per step} & \shortstack{Energy\\(mJ)} & Rel. \\
\midrule
SpikingBERT & rate & 16 & 25\%  & 750.8 & $10.50\times$ \\
SpikingLM   & rate & 4  & 33\%  & 247.4 & $3.46\times$ \\
Sorbet      & rate & 16 & 13\%  & 390.5 & $5.46\times$ \\
Otters      & TTFS & 15 & 5.14\% & 147.1 & $2.06\times$ \\
\midrule
\multirow{6}{*}{Matterhorn}
 & source QNN (dense)                & n/a & n/a & 71.54 & $1.00\times$ \\
 & bit-serial QNN ($k{=}0$ pattern)  & n/a & n/a & 201.6 & $2.82\times$ \\
 & bit-serial QNN ($k{=}1$ pattern)  & n/a & n/a & 119.4 & $1.67\times$ \\
 & TTFS                              & 16 & 3.85\% & 117.5 & $1.64\times$ \\
 & M-TTFS ($k{=}0$)                  & 16 & 2.52\% & 77.19 & $1.08\times$ \\
 & M-TTFS ($k{=}1$)                  & 16 & 1.57\% & 48.27 & $0.67\times$ \\
\bottomrule
\end{tabular}}
\end{table}

Tables~\ref{tab:spike_rate} and~\ref{tab:energy} report the spike rate and energy consumption. Standard TTFS fires for almost every activation on symmetric tensors because only $0.19\%$ of activations use its silent codeword. Its average per-step spike rate therefore remains $3.74\%$ ($s{=}0.60$). Reassigning silence to zero ($k{=}0$) removes $31\%$ of all events with no accuracy loss. Expanding the dead zone to $k{=}1$ further reduces the rate to $1.64\%$ per step ($s{=}0.26$), which is $25\%$ below the crossover threshold. It is the only evaluated encoding below this threshold, both on average and on every individual task (Table~\ref{tab:spike_rate}).

On SST-2, standard TTFS consumes $64\%$ more energy than its source QNN (Table~\ref{tab:energy}). Reassigning silence reduces this gap to $8\%$ at $k{=}0$, while $k{=}1$ lowers energy below the QNN: $48.27$ versus $71.54$\,mJ, giving a $1.48\times$ energy advantage over the dense INT4 network from which it was converted. It uses $67.2\%$ less energy than Spiking Otters.

Beyond dense streaming, we evaluate an event-stream
bit-serial QNN inspired by Stripes~\cite{judd2016stripes}
and Neural Cache~\cite{eckert2018neural}. Activations mapped
to zero are skipped, and each active bit of the remaining
signed code emits one routed event, with the bit position
encoded by the execution cycle. The bit-serial QNN and
M-TTFS use the same source-address cost. Under the $k{=}1$ pattern, the event-stream bit-serial QNN
consumes $119.4$\,mJ, compared with $48.27$\,mJ for
M-TTFS. Its traffic scales with the codeword Hamming weight,
whereas M-TTFS uses at most one routed event per non-silent
value. This comparison isolates the benefit of temporal
magnitude coding under the evaluated event-stream mapping.

\begin{figure}[htbp]
    \centering
    \includegraphics[width=0.9\linewidth]{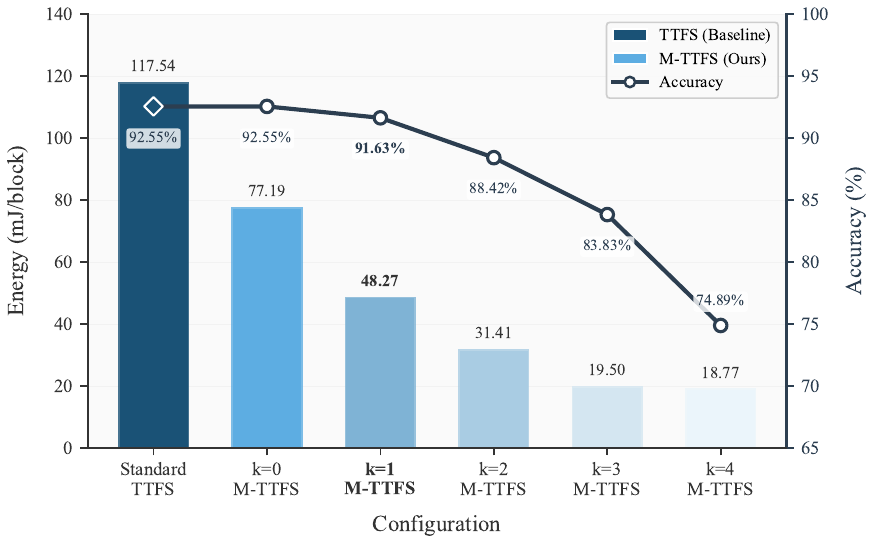}
    \caption{Energy-Accuracy Trade-off Analysis.}
    \label{fig:tradeoff}
\end{figure}

\subsection{Analysis and Ablations}
\label{sec:ablations}

\subsubsection{Zero Is Where Silence Belongs}
\label{sec:imax}

\begin{table}[t]
\centering
\caption{Proportion of the most frequent quantized value at each spike-encoded position on SST-2. In the source QNN, $q{=}0$ is the unique most frequent value at every position. Full spike-time distributions are shown in Appendix~\ref{app:mode_overview}.}
\label{tab:silent_mass}
\small
\begin{tabular}{@{}llc@{}}
\toprule
Position & Family & $p(q{=}0)$ \\
\midrule
Q-proj.\ input      & sym  & 0.264 \\
K-proj.\ input      & sym  & 0.264 \\
V-proj.\ input      & sym  & 0.293 \\
Query activation    & sym  & 0.200 \\
Attn.\ probability  & asym & 0.807 \\
Attn.-out.\ input   & sym  & 0.878 \\
FFN-expand input    & sym  & 0.330 \\
FFN-contract input  & asym & 0.911 \\
\bottomrule
\end{tabular}
\end{table}

Proposition~\ref{prop:optimal_mask} shows that the best silent codeword is the most common code. Table~\ref{tab:silent_mass} shows that this code is always zero. After combining each tensor position across the twelve encoder blocks, zero is the most common quantized value for every spike-encoded position on every task: all $56/56$ position-task pairs, with no exceptions. Depending on the position, zero covers $0.20$ to $0.91$ of all activations. Thus, the silent step of M-TTFS ($t{=}7$ for symmetric tensors and $t{=}15$ for asymmetric tensors) lies at the peak of every spike-time distribution.

The picture holds even below the position level. Across the
$8\times12{=}96$ per-block histograms on SST-2, 73 peak exactly
at zero and every one of the 23 exceptions peaks at the adjacent
code $-1$, a near tie that aggregation resolves to zero. Zero
therefore remains the shared global centre without any layerwise
calibration, and for $k{\ge}1$ the dead zone contains both codes,
so the tie is immaterial to the mask.
Appendix~\ref{app:mode_overview} shows the full distributions
and the per-block matrix.

\subsubsection{Dead-Zone Accuracy--Energy Trade-off}
\label{sec:dial}

Widening the dead zone yields a monotonic trade-off between accuracy and silence (Figure~\ref{fig:tradeoff}). Increasing $k$ from $0$ to $1$ reduces the GLUE average by only $1.23$ points while lowering energy by $37\%$. The full SST-2 sweep shows the same trend: as $k$ increases from $0$ to $4$, accuracy drops from $92.55$ to $74.89$, while the per-step spike rate decreases from $2.52$ to $0.60$. The best trade-off occurs at $k{=}1$; larger dead zones provide diminishing reductions in spike rate at a rapidly increasing accuracy cost. 

\subsubsection{Robustness to Sparse spike movement energy}
\label{sec:sensitivity}

Whether M-TTFS is more efficient than the dense QNN depends mainly on the cost of each routed sparse event. In the main results, we use $3$\,pJ per event based on Loihi~\cite{davies2018loihi}. This equals the cost of moving twelve dense bits and includes routing and address overhead, making it a conservative upper bound. We use the cost of one dense bit, $0.25$\,pJ, as the lower bound because an event cannot cost less than the bit it replaces. Even under this conservative setting, M-TTFS with $k{=}1$ remains more energy efficient than the dense QNN.

However, specialized SNN accelerators, such as LoAS~\cite{yin2024loas} and SpikeX~\cite{xu2025spikex}, reduce the cost of sparse-event movement. To cover these designs, Figure~\ref{fig:sensitivity} sweeps the evaluated range of sparse-event movement costs and shows that M-TTFS remains effective across all settings. M-TTFS with $k{=}1$ remains more efficient than the dense QNN at every event cost, ranging from a $10\times$ advantage at the lower bound to a $1.48\times$ advantage at the upper bound. The lossless setting, $k{=}0$, is more efficient for event costs below $2.77$\,pJ, covering over $90\%$ of the tested range, and is only $8\%$ less efficient at the upper bound.
\begin{figure}[t]
\centering
\includegraphics[width=0.85\linewidth]{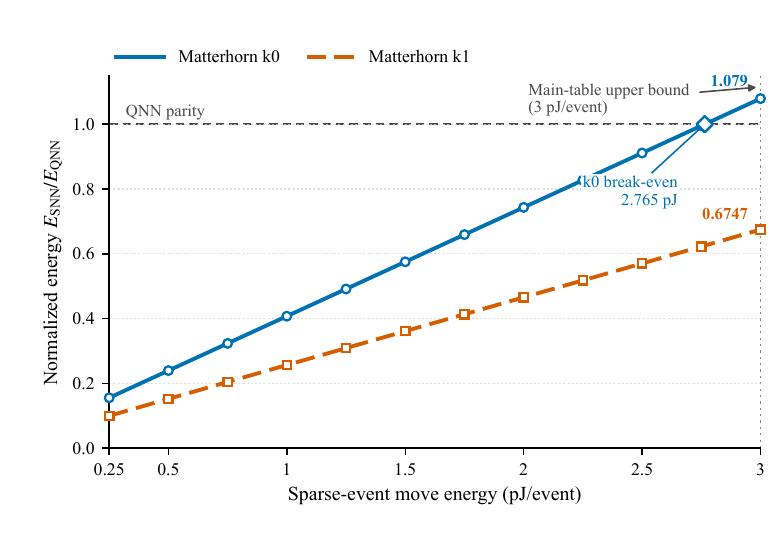}
\caption{Energy relative to the source QNN as the sparse-event
movement cost varies from $0.25$ to $3$\,pJ.
}
\label{fig:sensitivity}
\end{figure}

\subsection{Extended results for Larger model}

Finally, we apply M-TTFS to spiking LLaMA models from 7B to 70B
converted from 4-bit-weight, 4-bit-activation
QNNs~\cite{wang2026kirinimprovingannefficiency}
(Table~\ref{tab:llm}). The premise transfers before the method
does: zero is the most frequent code at every spike-encoded
position and every scale, so the same global $\mu^{\ell}{=}0$
rule applies with no per-layer or per-scale search (per-position
distributions in Appendix~\ref{app:llm}). Unlike the BERT track,
no fine-tuning is used at any radius: every row of
Table~\ref{tab:llm} is a training-free post-training conversion,
which is what allows the method to reach 70B parameters at low
cost. TTFS and $k{=}0$ are bit-exact to the source QNN; $k{=}0$
already removes 16\% to 17\% of all events for free, and $k{=}1$
removes 42\% to 44\% at a perplexity cost that shrinks with
scale, from $0.36$ at 7B to $0.24$ at 70B on Wikitext-2. Movement energy falls
in proportion to the rate, consistently across all three scales.
Pushing to $k{=}2$ without the fine-tuning repair used on BERT
degrades perplexity sharply, though less so at larger scales,
mirroring Section~\ref{sec:dial}: the dead zone remains a dial,
and training is the price of turning it far.

\begin{table}[t]
\centering
\caption{Spiking LLaMA at three scales. TTFS and M-TTFS
($k{=}0$) are both lossless and bit-exact to the 4-bit source
QNN, hence identical perplexity; all rows are training-free
post-training conversions. The spike rate is averaged over
WikiText-2 and C4; rate reduction is relative to TTFS and computed from unrounded spike rates.}
\label{tab:llm}
\small
\scalebox{0.8}{
\begin{tabular}{@{}llcccc@{}}
\toprule
Model & Encoding & \shortstack{WikiText-2\\PPL} & \shortstack{C4\\PPL} & \shortstack{Spike Rate\\per step} & \shortstack{Rate\\reduction} \\
\midrule
\multirow{4}{*}{7B}  & TTFS    & 5.74  & 7.27  & 6.18\% & n/a \\
                     & $k{=}0$ & 5.74  & 7.27  & 5.19\% & 16.0\% \\
                     & $k{=}1$ & 6.10  & 7.66  & 3.57\% & 42.3\% \\
                     & $k{=}2$ & 13.25 & 16.17 & 2.33\% & 62.3\% \\
\midrule
\multirow{4}{*}{13B} & TTFS    & 5.04  & 6.64  & 6.18\% & n/a \\
                     & $k{=}0$ & 5.04  & 6.64  & 5.20\% & 15.9\% \\
                     & $k{=}1$ & 5.32  & 6.91  & 3.58\% & 42.1\% \\
                     & $k{=}2$ & 8.94  & 11.49 & 2.34\% & 62.1\% \\
\midrule
\multirow{4}{*}{70B} & TTFS    & 3.49  & 5.65  & 6.17\% & n/a \\
                     & $k{=}0$ & 3.49  & 5.65  & 5.10\% & 17.4\% \\
                     & $k{=}1$ & 3.73  & 5.82  & 3.48\% & 43.7\% \\
                     & $k{=}2$ & 5.40  & 7.49  & 2.30\% & 62.9\% \\
\bottomrule
\end{tabular}}
\end{table}

\section{Conclusion}
In this paper, we introduced Matterhorn, a high-performance
spiking transformer designed to overcome the energy inefficiency caused by data movement in SNNs. The energy advantage of SNNs is not automatic. Once data
movement is charged, it must be earned by the encoding. This
paper showed how to earn it. M-TTFS reassigns the only cost-free
codeword, silence, to the most frequent activation value. This
assignment minimizes the expected spike count and preserves the source
quantized network exactly.
A dead-zone extension trades a controlled amount of accuracy for
further sparsity. Built on this encoding, Matterhorn sets a new
SOTA for spiking transformers on GLUE in terms of accuracy and energy. The same rule further transfers to spiking LLaMA from
7B to 70B parameters. Spikes are not cheap
because they are sparse. They are cheap only when silence sits
where the data actually lives.

\newpage

\bibliography{aaai2027}

@article{yan2024reconsidering,
  title={Reconsidering the energy efficiency of spiking neural networks},
  author={Yan, Zhanglu and Bai, Zhenyu and Wong, Weng-Fai},
  journal={arXiv preprint arXiv:2409.08290},
  year={2024}
}

@article{wang2025energy,
  title={Energy-Efficient and Dequantization-Free {Q-LLM}s: A Spiking Neural Network Approach to Salient Value Mitigation},
  author={Wang, Chenyu and Yan, Zhanglu and Zhou, Zhi and Chen, Xu and Wong, Weng-Fai},
  journal={arXiv preprint arXiv:2510.19498},
  year={2025}
}

@String{Computing = "Computing" }

@String{Computer = "{IEEE} Computer" }

@inproceedings{
zhao2025ttfsformer,
title={{TTFSF}ormer: A {TTFS}-based Lossless Conversion of Spiking Transformer},
author={Lusen Zhao and Zihan Huang and Jianhao Ding and Zhaofei Yu},
booktitle={Forty-second International Conference on Machine Learning},
year={2025},
url={https://openreview.net/forum?id=mJAa823xKu}
}

@inproceedings{wei2023temporal,
  title={Temporal-coded spiking neural networks with dynamic firing threshold: Learning with event-driven backpropagation},
  author={Wei, Wenjie and Zhang, Malu and Qu, Hong and Belatreche, Ammar and Zhang, Jian and Chen, Hong},
  booktitle={Proceedings of the IEEE/CVF international conference on computer vision},
  pages={10552--10562},
  year={2023}
}

@article{davies2018loihi,
  title={Loihi: A neuromorphic manycore processor with on-chip learning},
  author={Davies, Mike and Srinivasa, Narayan and Lin, Tsung-Han and Chinya, Gautham and Cao, Yongqiang and Choday, Sri Harsha and Dimou, Georgios and Joshi, Prasad and Imam, Nabil and Jain, Shweta and others},
  journal={Ieee Micro},
  volume={38},
  number={1},
  pages={82--99},
  year={2018},
  publisher={IEEE}
}

@inproceedings{bal2024spikingbert,
  title={Spikingbert: Distilling bert to train spiking language models using implicit differentiation},
  author={Bal, Malyaban and Sengupta, Abhronil},
  booktitle={Proceedings of the AAAI conference on artificial intelligence},
  volume={38},
  number={10},
  pages={10998--11006},
  year={2024}
}

@inproceedings{judd2016stripes,
  title={Stripes: Bit-serial deep neural network computing},
  author={Judd, Patrick and Albericio, Jorge and Hetherington, Tayler and Aamodt, Tor M and Moshovos, Andreas},
  booktitle={2016 49th Annual IEEE/ACM International Symposium on Microarchitecture (MICRO)},
  pages={1--12},
  year={2016},
  organization={IEEE}
}

@inproceedings{eckert2018neural,
  title={Neural cache: Bit-serial in-cache acceleration of deep neural networks},
  author={Eckert, Charles and Wang, Xiaowei and Wang, Jingcheng and Subramaniyan, Arun and Iyer, Ravi and Sylvester, Dennis and Blaaauw, David and Das, Reetuparna},
  booktitle={2018 ACM/IEEE 45Th annual international symposium on computer architecture (ISCA)},
  pages={383--396},
  year={2018},
  organization={IEEE}
}

@misc{wang2026kirinimprovingannefficiency,
      title={Kirin: Improving ANN efficiency with SNN Hybridization}, 
      author={Chenyu Wang and Zhanglu Yan and Zhi Zhou and Xu Chen and Weng-Fai Wong},
      year={2026},
      eprint={2602.08817},
      archivePrefix={arXiv},
      primaryClass={cs.LG},
      url={https://arxiv.org/abs/2602.08817}, 
}

@inproceedings{yin2024loas,
  title={Loas: Fully temporal-parallel dataflow for dual-sparse spiking neural networks},
  author={Yin, Ruokai and Kim, Youngeun and Wu, Di and Panda, Priyadarshini},
  booktitle={2024 57th IEEE/ACM International Symposium on Microarchitecture (MICRO)},
  pages={1107--1121},
  year={2024},
  organization={IEEE}
}

@article{xu2025spikex,
  title={Spikex: Exploring accelerator architecture and network-hardware co-optimization for sparse spiking neural networks},
  author={Xu, Boxun and Boone, Richard and Li, Peng},
  journal={IEEE Transactions on Computer-Aided Design of Integrated Circuits and Systems},
  year={2025},
  publisher={IEEE}
}

@inproceedings{devlin2019bert,
  title={Bert: Pre-training of deep bidirectional transformers for language understanding},
  author={Devlin, Jacob and Chang, Ming-Wei and Lee, Kenton and Toutanova, Kristina},
  booktitle={Proceedings of the 2019 conference of the North American chapter of the association for computational linguistics: human language technologies, volume 1 (long and short papers)},
  pages={4171--4186},
  year={2019}
}

@article{liu2022bit,
  title={Bit: Robustly binarized multi-distilled transformer},
  author={Liu, Zechun and Oguz, Barlas and Pappu, Aasish and Xiao, Lin and Yih, Scott and Li, Meng and Krishnamoorthi, Raghuraman and Mehdad, Yashar},
  journal={Advances in neural information processing systems},
  volume={35},
  pages={14303--14316},
  year={2022}
}

@inproceedings{jiao2020tinybert,
  title={TinyBERT: Distilling BERT for Natural Language Understanding},
  author={Jiao, Xiaoqi and Yin, Yichun and Shang, Lifeng and Jiang, Xin and Chen, Xiao and Li, Linlin and Wang, Fang and Liu, Qun},
  booktitle={Findings of the Association for Computational Linguistics: EMNLP 2020},
  pages={4163--4174},
  year={2020}
}

@article{sanh2019distilbert,
  title={DistilBERT, a distilled version of BERT: smaller, faster, cheaper and lighter},
  author={Sanh, Victor and Debut, Lysandre and Chaumond, Julien and Wolf, Thomas},
  journal={arXiv preprint arXiv:1910.01108},
  year={2019}
}

@article{bai2025data,
  title={A Data-Driven Dynamic Execution Orchestration Architecture},
  author={Bai, Zhenyu and Dangi, Pranav and Juneja, Rohan and Li, Zhaoying and Yan, Zhanglu and Lan, Huiying and Mitra, Tulika},
  journal={International Conference on Architectural Support for Programming Languages and Operating Systems (ASPLOS)},
  year={2026},
}

@inproceedings{prabhakar2022sambanova,
  title={SambaNova SN10 RDU: A 7nm dataflow architecture to accelerate software 2.0},
  author={Prabhakar, Raghu and Jairath, Sumti and Shin, Jinuk Luke},
  booktitle={2022 IEEE International Solid-State Circuits Conference (ISSCC)},
  volume={65},
  pages={350--352},
  year={2022},
  organization={IEEE}
}

@inproceedings{xingspikelm,
  title={SpikeLM: Towards General Spike-Driven Language Modeling via Elastic Bi-Spiking Mechanisms},
  author={Xing, Xingrun and Zhang, Zheng and Ni, Ziyi and Xiao, Shitao and Ju, Yiming and Fan, Siqi and Wang, Yequan and Zhang, Jiajun and Li, Guoqi},
  booktitle={Forty-first International Conference on Machine Learning},
year = {2024}
}

@article{xing2024spikellm,
  title={Spikellm: Scaling up spiking neural network to large language models via saliency-based spiking},
  author={Xing, Xingrun and Gao, Boyan and Zhang, Zheng and Clifton, David A and Xiao, Shitao and Du, Li and Li, Guoqi and Zhang, Jiajun},
  journal={arXiv preprint arXiv:2407.04752},
  year={2024}
}

@inproceedings{tangsorbet,
  title={Sorbet: A Neuromorphic Hardware-Compatible Transformer-Based Spiking Language Model},
  author={Tang, Kaiwen and Yan, Zhanglu and Wong, Weng-Fai},
  booktitle={Forty-second International Conference on Machine Learning},
year={2025}
}

@misc{zhou2025spikingformerkeyfoundationmodel,
      title={Spikingformer: A Key Foundation Model for Spiking Neural Networks}, 
      author={Chenlin Zhou and Liutao Yu and Zhaokun Zhou and Han Zhang and Jiaqi Wang and Huihui Zhou and Zhengyu Ma and Yonghong Tian},
      year={2025},
      eprint={2304.11954},
      archivePrefix={arXiv},
      primaryClass={cs.NE},
      url={https://arxiv.org/abs/2304.11954}, 
}

@inproceedings{
yan2025otters,
title={Otters: An Energy-Efficient Spiking Transformer via Optical Time-to-First-Spike Encoding},
author={Zhanglu Yan and Jiayi Mao and Qianhui Liu and Fanfan Li and Tao Luo and Gang Pan and Bowen Zhu and Weng-Fai Wong},
booktitle={The Fourteenth International Conference on Learning Representations},
year={2026},
url={https://openreview.net/forum?id=oK0ISeb5Dw}
}

@article{sakemi2023sparse,
  title   = {Sparse-firing regularization methods for spiking neural networks with time-to-first-spike coding},
  author  = {Sakemi, Yusuke and Yamamoto, Kakei and Hosomi, Takeo and Aihara, Kazuyuki},
  journal = {Scientific Reports},
  volume  = {13},
  number  = {1},
  pages   = {22897},
  year    = {2023},
  doi     = {10.1038/s41598-023-50201-5}
}

@inproceedings{sun2025temporal,
author = {Sun, Qian and Lu, Chengzhuo and Chen, Wenyu and Wei, Wenjie and Wang, Jingya and Zhang, Jieyuan and Liu, Xiaoli and Ye, Yalan and Yang, Yang and Zhang, Malu},
title = {Temporal-coded Spiking Transformer},
year = {2025},
doi = {10.1145/3746027.3754545},
booktitle = {Proceedings of the 33rd ACM International Conference on Multimedia},
pages = {2616–2624},
numpages = {9},
}

@inproceedings{lew2022ttfstraining,
author = {Lew, Dongwoo and Lee, Kyungchul and Park, Jongsun},
title = {A time-to-first-spike coding and conversion aware training for energy-efficient deep spiking neural network processor design},
year = {2022},
doi = {10.1145/3489517.3530457},
booktitle = {Proceedings of the 59th ACM/IEEE Design Automation Conference},
pages = {265–270},
numpages = {6},
}

\newpage
\appendix

\newpage
\appendix

\setcounter{topnumber}{4}
\setcounter{bottomnumber}{3}
\setcounter{totalnumber}{6}
\setcounter{dbltopnumber}{4}
\renewcommand{\topfraction}{0.94}
\renewcommand{\bottomfraction}{0.94}
\renewcommand{\textfraction}{0.06}
\renewcommand{\floatpagefraction}{0.90}
\renewcommand{\dbltopfraction}{0.94}
\renewcommand{\dblfloatpagefraction}{0.90}
\makeatletter
\floatsep=10pt plus 3pt minus 2pt
\textfloatsep=12pt plus 3pt minus 2pt
\dblfloatsep=10pt plus 3pt minus 2pt
\dbltextfloatsep=12pt plus 3pt minus 2pt
\@fptop=0pt      \@fpsep=12pt     \@fpbot=0pt plus 1fil
\@dblfptop=0pt   \@dblfpsep=12pt  \@dblfpbot=0pt plus 1fil
\makeatother
\captionsetup{skip=5pt}

\section{Full Spike-Time Distributions and Mode Statistics}
\label{app:mode_overview}

This appendix reports the complete distributional evidence behind
Section~\ref{sec:imax}.

\begin{table*}[t]
\centering
\caption{Mode location $\hat q$ for all eleven tensor positions
on all seven tasks (source QNN).}
\label{tab:mode_overview}
\small
\setlength{\tabcolsep}{3.5pt}
\begin{tabular}{@{}lcccccccc@{}}
\toprule
Position & Fam. & QQP & MNLI & SST-2 & QNLI & RTE & MRPC & STS-B \\
\midrule
Q-proj.\ input     & sym  & 0 & 0 & 0 & 0 & 0 & 0 & 0 \\
K-proj.\ input     & sym  & 0 & 0 & 0 & 0 & 0 & 0 & 0 \\
V-proj.\ input     & sym  & 0 & 0 & 0 & 0 & 0 & 0 & 0 \\
Query act.         & sym  & 0 & 0 & 0 & 0 & 0 & 0 & 0 \\
Attn.\ prob.       & asym & 0 & 0 & 0 & 0 & 0 & 0 & 0 \\
Attn.-out.\ input  & sym  & 0 & 0 & 0 & 0 & 0 & 0 & 0 \\
FFN-expand input   & sym  & 0 & 0 & 0 & 0 & 0 & 0 & 0 \\
FFN-contract input & asym & 0 & 0 & 0 & 0 & 0 & 0 & 0 \\
\bottomrule
\end{tabular}
\end{table*}

\begin{table*}[t]
\centering
\begin{minipage}[t]{0.44\textwidth}
\centering
\caption{Mode of every spike-encoded position after dead-zone
training or projection, per radius. All $80$ position and
configuration pairs stay inside their frozen dead zone. The only
modes outside a zone anywhere in the audit belong to the value
activation on QQP and STS-B, a tensor that is never spike-encoded
and therefore never masked.}
\label{tab:finetuned_modes}
\small
\begin{tabular}{@{}lll@{}}
\toprule
Radius & Positions with mode off zero & Inside zone \\
\midrule
$k{=}1$ (7 tasks) & none & $56/56$ \\
$k{=}2$ (SST-2)   & Q-proj., K-proj.\ at $-1$ & $8/8$ \\
$k{=}3$ (SST-2)   & none & $8/8$ \\
$k{=}4$ (SST-2)   & none & $8/8$ \\
\bottomrule
\end{tabular}
\end{minipage}\hfill
\begin{minipage}[t]{0.53\textwidth}
\centering
\caption{Per-block mode of each spike-encoded position on SST-2
(source QNN). $73$ of $96$ cells peak exactly at zero and all
$23$ deviations sit at the adjacent code $-1$, a systematic near
tie between $q{=}0$ and $q{=}-1$ in the deeper blocks of the
projection inputs. Aggregation over blocks resolves every
position to zero (Table~\ref{tab:mode_overview}), and for
$k{\ge}1$ both codes lie inside the dead zone, so the tie has no
effect on the mask.}
\label{tab:block_modes}
\small
\setlength{\tabcolsep}{1.6pt}
\begin{tabular}{@{}lcccccccccccc@{}}
\toprule
Position & 0 & 1 & 2 & 3 & 4 & 5 & 6 & 7 & 8 & 9 & 10 & 11 \\
\midrule
Q-proj.\ input     & 0 & 0 & 0 & $-1$ & $-1$ & $-1$ & $-1$ & $-1$ & $-1$ & $-1$ & $-1$ & $-1$ \\
K-proj.\ input     & 0 & 0 & $-1$ & $-1$ & $-1$ & $-1$ & $-1$ & $-1$ & $-1$ & $-1$ & $-1$ & $-1$ \\
V-proj.\ input     & 0 & 0 & 0 & 0 & 0 & 0 & 0 & 0 & $-1$ & $-1$ & 0 & $-1$ \\
Query act.         & 0 & 0 & 0 & 0 & 0 & 0 & 0 & 0 & 0 & 0 & 0 & $-1$ \\
Attn.\ prob.       & 0 & 0 & 0 & 0 & 0 & 0 & 0 & 0 & 0 & 0 & 0 & 0 \\
Attn.-out.\ input  & 0 & 0 & 0 & 0 & 0 & 0 & 0 & 0 & 0 & 0 & 0 & 0 \\
FFN-expand input   & 0 & 0 & 0 & 0 & 0 & 0 & 0 & 0 & 0 & 0 & 0 & 0 \\
FFN-contract input & 0 & 0 & 0 & 0 & 0 & 0 & 0 & 0 & 0 & 0 & 0 & 0 \\
\bottomrule
\end{tabular}
\end{minipage}
\end{table*}

\begin{figure*}[t]
\centering
\includegraphics[width=0.88\textwidth]{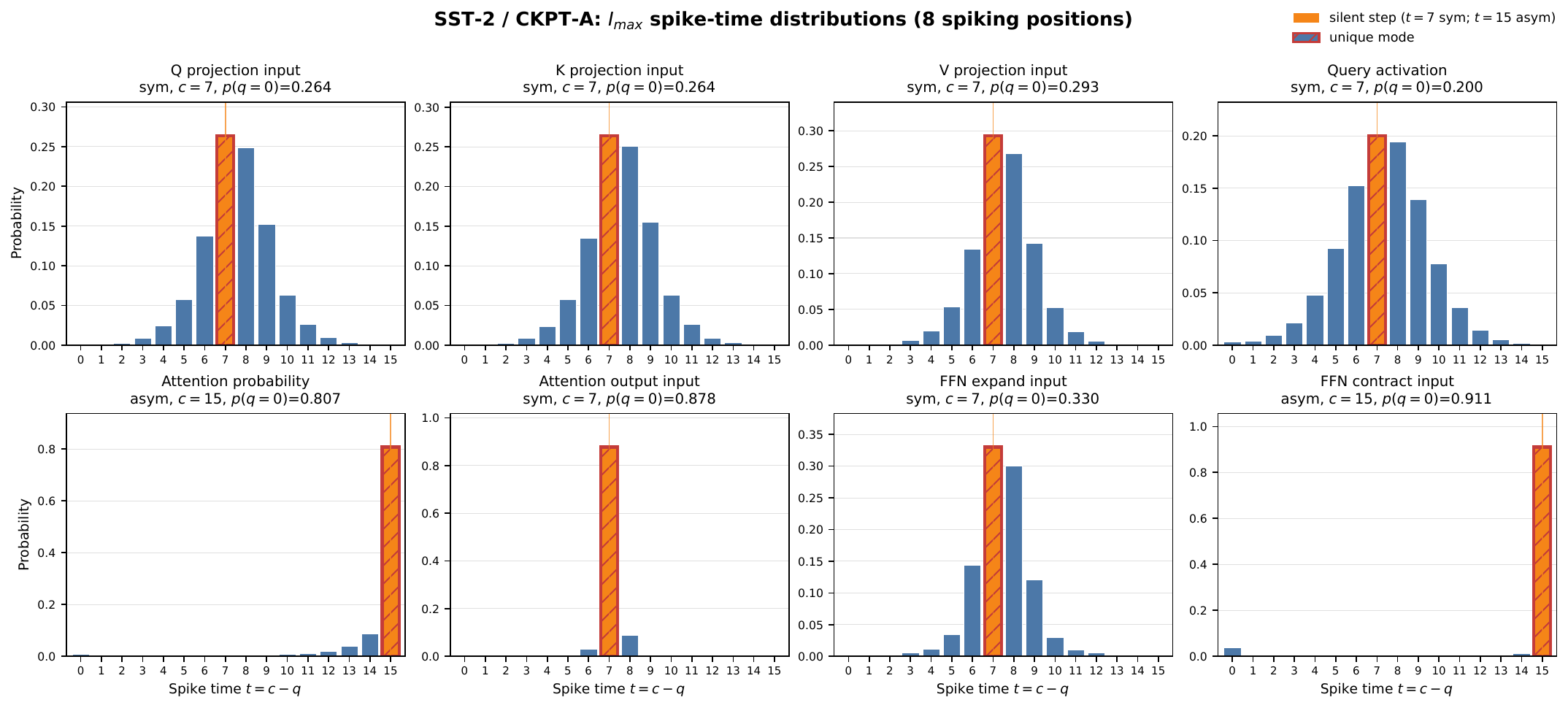}
\caption{Aggregated spike-time distributions of the eight
spike-encoded tensor positions (SST-2, source QNN), with event
time $t=c-q$ and $p(q{=}0)$ annotated per panel. The silent step
of M-TTFS (orange, $t{=}7$ for symmetric and $t{=}15$ for
asymmetric tensors) coincides with the unique mode (hatched) of
every position.}
\label{fig:position_dists}
\end{figure*}

\begin{figure*}[t]
\centering
\includegraphics[width=0.88\textwidth]{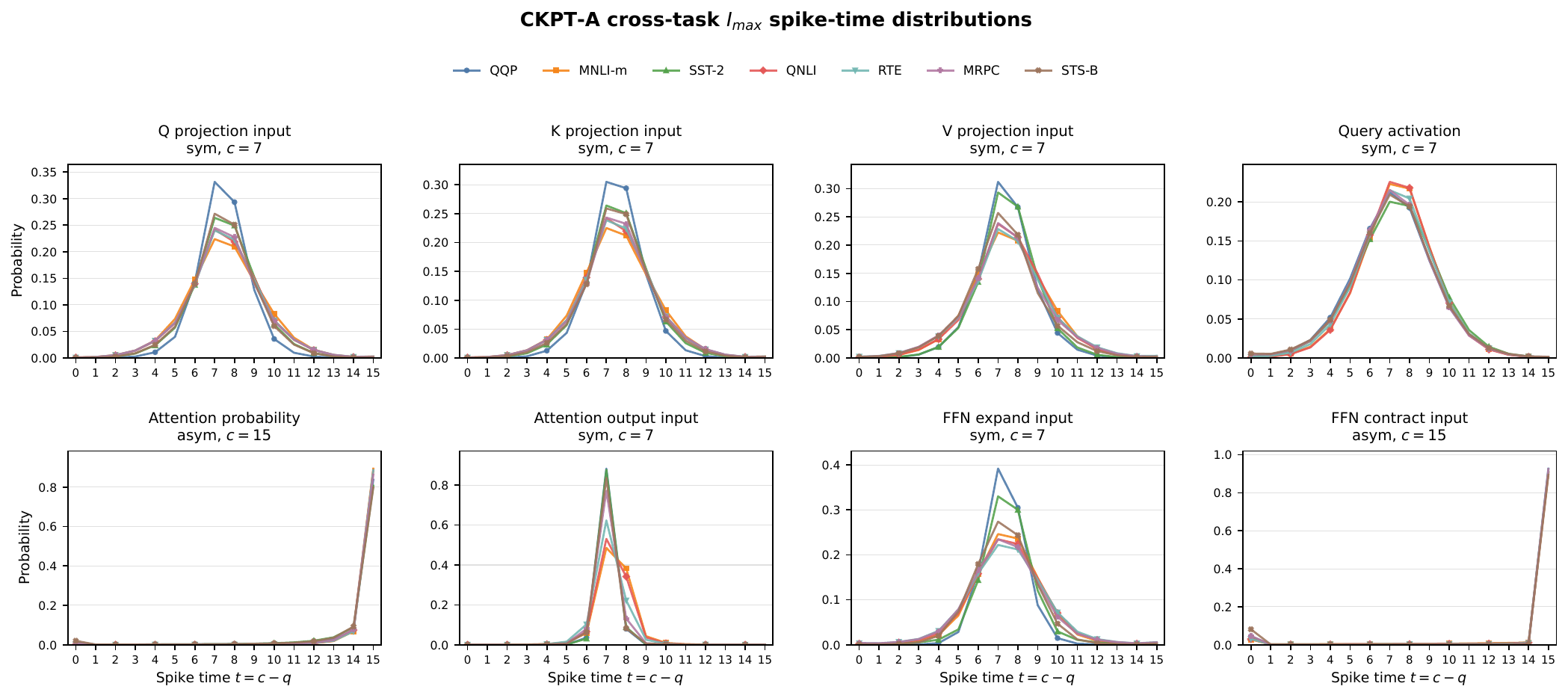}
\caption{The same eight positions with all seven GLUE tasks
overlaid. The shapes are nearly identical across tasks and every
curve peaks at the silent step. QQP is the sharpest on the
projection inputs, which is why it has the lowest spike rates in
Table~\ref{tab:spike_rate}.}
\label{fig:crosstask_overlay}
\end{figure*}

Three observations complete the picture. First, on the source
network the mode of every spike-encoded position is exactly zero
on every task, and every deviation in the full audit of
Table~\ref{tab:mode_overview} is confined to the three densely
transmitted positions. Second, the choice is stable under
training. Table~\ref{tab:finetuned_modes} audits every dead-zone
checkpoint and finds all $80$ spike-encoded modes inside their
frozen zones, with the $k{=}2$ drift of the projection inputs to
$-1$ illustrating that movement inside the zone is functionally
free, since the projection identifies all codes it covers. Third,
the finest granularity is honest about its limits.
Table~\ref{tab:block_modes} shows that individual blocks can
prefer the adjacent code $-1$, yet the position aggregate always
selects zero, and Section~\ref{sec:imax} shows that calibrating
at any finer granularity changes the spike rate by exactly zero
at the granularity the method uses.

\section{Spiking LLaMA Details}
\label{app:llm}

All large-model results use LLaMA-2 at 7B, 13B, and 70B,
converted from 4-bit-weight, 4-bit-activation
QNNs~\cite{wang2026kirinimprovingannefficiency}. The conversion
is pure post-training at every radius: the dead zone is applied
as a projection and no fine-tuning is performed anywhere in the
pipeline. 

Figure~\ref{fig:llm_dist} shows the aggregated quantized
activation distribution. It is peaked at $q{=}0$, nearly
symmetric, and scale invariant, with a maximum total-variation
distance below $2\%$ between any two model scales on either
corpus. Figure~\ref{fig:llm_granularity} refines this picture by
granularity. Zero is the unique peak at every level from the
whole network down to individual quantization positions, all at
$100\%$. At the finest level, individual attention heads, $96.5\%$
of the $28{,}544$ heads peak at zero and $99.5\%$ peak within one
code of zero. A single global mask at $I_{\max}{=}7$ therefore
loses essentially nothing to any finer-grained calibration.
Figure~\ref{fig:llm_bars} reports the full per-corpus results
behind Table~\ref{tab:llm}. The two corpora agree to within
$0.01$ points of spike rate at every configuration, which
justifies the averaged rates in the main table.

\begin{figure*}[t]
\centering
\includegraphics[width=0.76\textwidth]{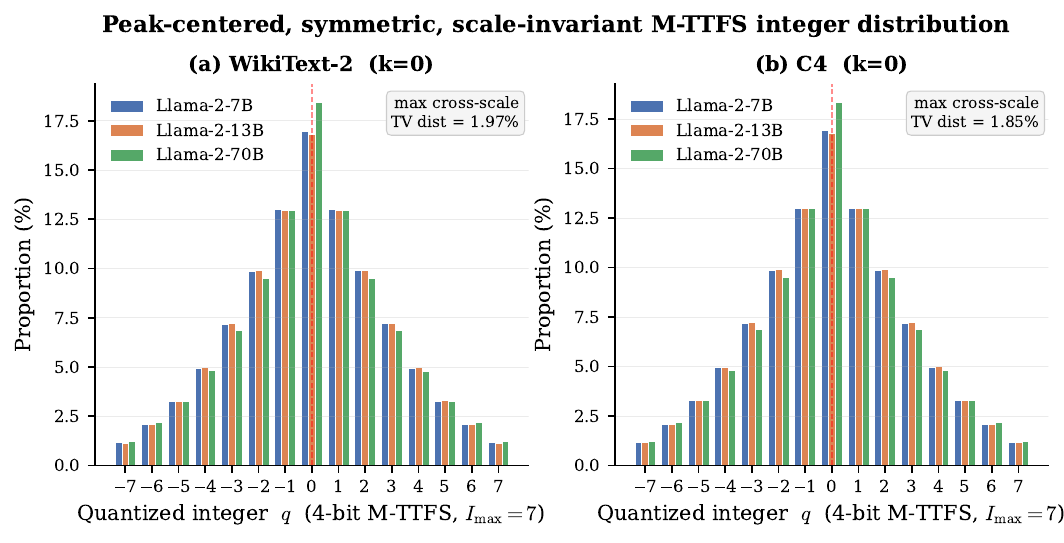}
\caption{Aggregated quantized activation distributions of spiking
LLaMA-2 at three scales on WikiText-2 and C4 ($k{=}0$). The peak
at $q{=}0$ corresponds to the silent step $t{=}I_{\max}{=}7$
under the bijection of Eq.~\eqref{eq:q_to_time}. The maximum
total-variation distance between scales is $1.97\%$ on WikiText-2
and $1.85\%$ on C4.}
\label{fig:llm_dist}
\end{figure*}

\begin{figure*}[t]
\centering
\includegraphics[width=0.86\textwidth]{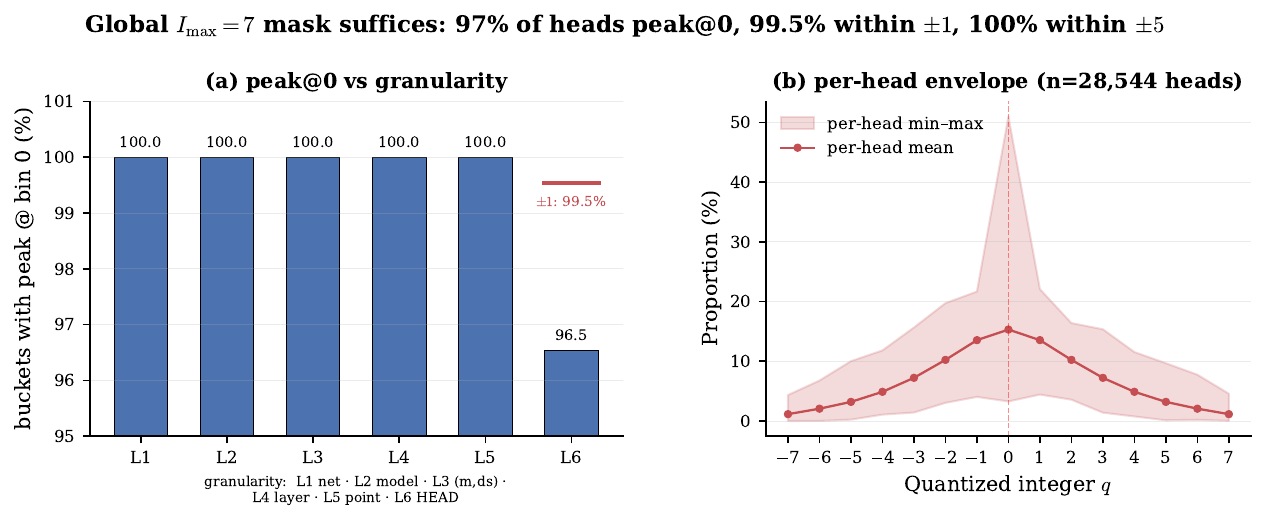}
\caption{Mode location by statistical granularity, pooled over
scales and corpora ((L1 whole network, L2 per scale, L3 per scale and corpus, L4 per
layer, L5 per quantization position, L6 per attention head)). Zero is the peak for $100\%$ of buckets at
every level from the whole network (L1) down to individual
quantization positions (L5). At the per-head level (L6,, $96.5\%$ of heads peak at zero and $99.5\%$ peak
within one code of zero. The right panel shows the per-head mean
distribution and its min-max envelope.}
\label{fig:llm_granularity}
\end{figure*}

\begin{figure*}[t]
\centering
\includegraphics[width=0.74\textwidth]{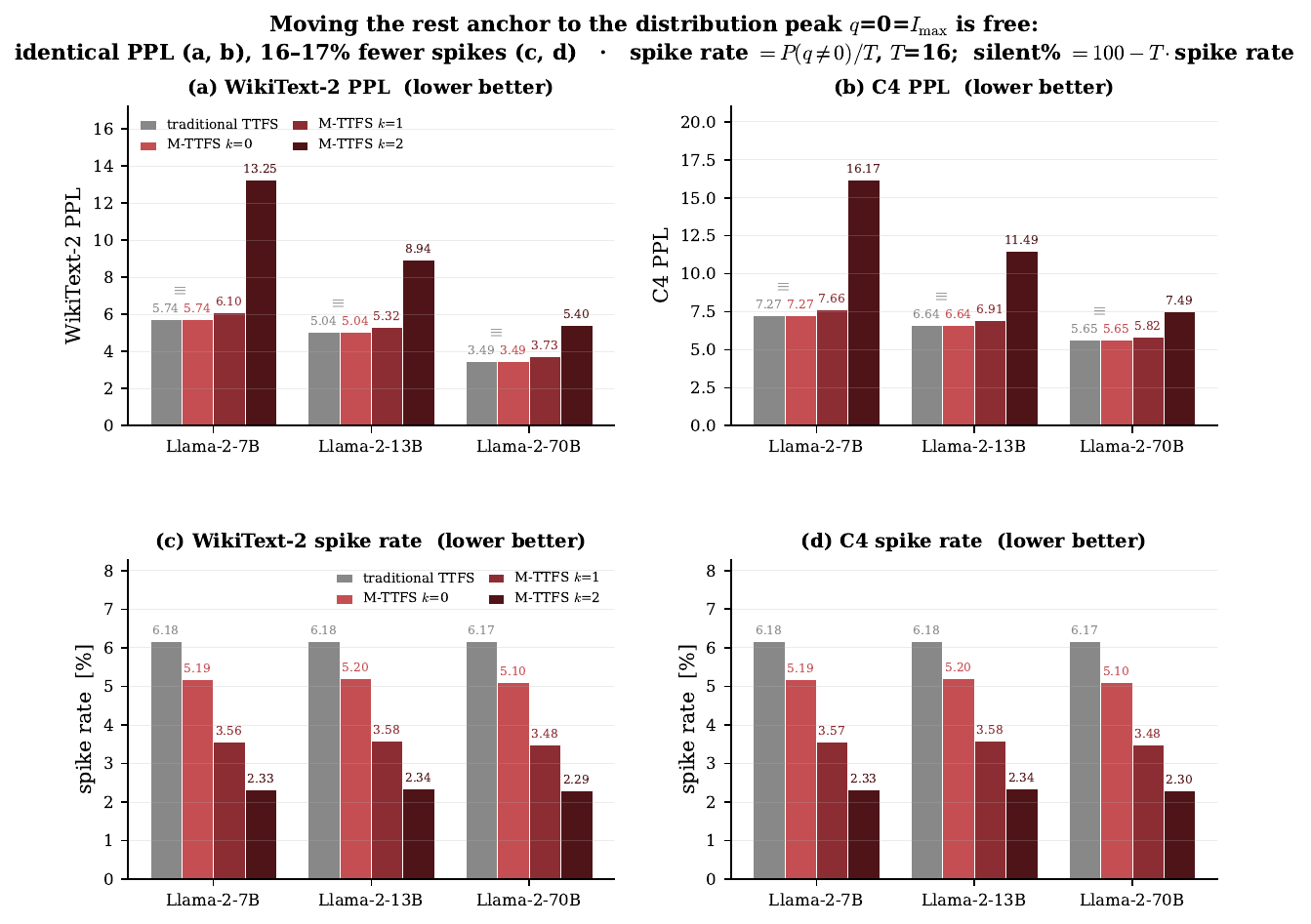}
\caption{Per-corpus perplexity and spike rate behind
Table~\ref{tab:llm}. Moving silence to the peak ($k{=}0$) leaves
perplexity bit-exact while removing $16$ to $17\%$ of spikes;
larger radii trade perplexity for further sparsity. WikiText-2
and C4 rates agree to within $0.01$ points everywhere.}
\label{fig:llm_bars}
\end{figure*}


\section{Energy Model and Unit Costs}
\label{app:energy}

All energies in this paper are analytical estimates for one
encoder block processing one batch of $B{=}64$ sequences at
$S{=}128$: operation and delivery counts are multiplied by the
per-op unit costs of Table~\ref{tab:appendix_c1_unit_costs},
with the architectural parameters of
Table~\ref{tab:appendix_c1_formula_parameters}. The scope is
digital only custom Matterhorn
hardware. Movement is charged per operand delivery at one hop
per delivery. The fanout counts have closed forms. The eight
spike-encoded positions contribute
$N_8 = 4BSC^2 + 2BhS^2d_k + 2BSCF = 5.96\times10^{10}$
deliveries per block. 

\textbf{Provenance.} Three sources supply the unit costs. The arithmetic we
synthesized ourselves: each compute constant
($E_{\mathrm{ACC}}$, $E_{\mathrm{MAC}}$, $E_{\mathrm{CMP}}$,
$E_{\mathrm{SUB}}$) is the per-operator result of synthesizing
that operator on a commercial 22\,nm standard-cell library with
Design Compiler at the TT corner, 0.80\,V, 25\,$^\circ$C, and
1\,GHz under vectorless switching activity, and
$E_{\mathrm{weight}}$ is the per-bit figure of a low-power
single-port SRAM macro characterized by the memory compiler at
the same corner. Dense movement we measured ourselves:
$\overline{E}_{\mathrm{move}} = 0.25$\,pJ per bit per hop comes
from our measurement of a circuit-switched NoC at 22\,nm. The
one number we did not produce is the event price.
$\widetilde{E}_{\mathrm{move}} = 3.0$\,pJ per event per hop is
Loihi's published per-hop routing energy (14\,nm, 0.75\,V,
pre-silicon, Table~2 of~\cite{davies2018loihi}), adopted without
rescaling as a deliberately conservative cross-platform anchor
rather than as a measurement of this design. The price is a
bundle. It already covers payload, address, and routing, so the
model commits to no physical address width, and the factor of 12
relative to a dense bit is a cost ratio, not a packet format.
Because this borrowed anchor is the one constant not grounded at
our node, it is exactly the constant that
Section~\ref{sec:sensitivity} sweeps, from the physical floor of
one dense bit, since an event cannot cost less than the bit it
replaces, up to the Loihi value itself. Every main-table energy
therefore already charges M-TTFS the least favourable price in
that range. The artifact includes the synthesis scripts and
power reports behind these numbers.

\textbf{Bit-serial QNN accounting.} The bit-serial rows follow
the execution principles of Stripes~\cite{judd2016stripes} and
Neural Cache~\cite{eckert2018neural}. Under a given mask the
design skips every zero-valued activation; each survivor then
transmits only the nonzero bits of its magnitude, one indexed
transfer per bit at the sparse-event price
$\widetilde{E}_{\mathrm{move}}$, and its weight reads and
compute are gated on the same survivors. This gives the QNN both
value-level and bit-level sparsity under exactly the masks
M-TTFS uses, and Table~\ref{tab:appendix_c2_breakdown} shows it
still does not reach dense streaming, let alone M-TTFS.

\begin{table*}[t]
\centering
\caption{Digital unit costs of the 22\,nm analytical model.}
\label{tab:appendix_c1_unit_costs}
\begingroup
\setlength{\tabcolsep}{3pt}
\scriptsize
\begin{tabular}{p{0.20\textwidth}p{0.13\textwidth}p{0.09\textwidth}p{0.12\textwidth}p{0.40\textwidth}}
\toprule
Item & Symbol & Value & Unit & Notes \\
\midrule
INT4 MAC, 1-bit W $\times$ 4-bit A & $E_{\mathrm{MAC4}}$ & 0.06634 & pJ/op & QNN and TTFS linear paths. \\
4-bit $\times$ 4-bit MAC & $E_{\mathrm{MAC4\times4}}$ & 0.0848 & pJ/op & QNN attention path. \\
ACC, 1-bit & $E_{\mathrm{ACC1}}$ & 0.04292 & pJ/op & Rate-coded SNN linear layers. \\
ACC, 2-bit & $E_{\mathrm{ACC2}}$ & 0.0477 & pJ/op & $T{=}4$ rate-coded attention path. \\
ACC, 4-bit & $E_{\mathrm{ACC4}}$ & 0.05021 & pJ/op & $T{=}16$ rate-coded attention path. \\
Comparator & $E_{\mathrm{CMP}}$ & 0.05021 & pJ/op & Threshold comparison. \\
Membrane subtraction & $E_{\mathrm{SUB}}$ & 0.05021 & pJ/op & Decay update. \\
INT4 clamp & $E_{\mathrm{clamp4}}$ & 0.05021 & pJ/op & QNN quantization clamp. \\
SRAM read or write & $E_{\mathrm{weight}}$ & 0.09845 & pJ/bit & One per-bit value for reads and writes. \\
Dense routed transfer & $\overline{E}_{\mathrm{move}}$ & 0.25 & pJ/bit$\cdot$hop & No event or address overhead. \\
Sparse-event transfer & $\widetilde{E}_{\mathrm{move}}$ & 3.0 & pJ/event$\cdot$hop & Bundles payload, address, and routing; equals 12 dense-bit equivalents; not multiplied by a width. \\
Digital TTFS encoding & $E_{\mathrm{enc}}$ & 0.0163 & pJ/event & Otters and Matterhorn TTFS rows. \\
\bottomrule
\end{tabular}
\endgroup
\end{table*}

\begin{table*}[t]
\centering
\caption{Formula parameters and the sequence-length convention.}
\label{tab:appendix_c1_formula_parameters}
\begingroup
\setlength{\tabcolsep}{3pt}
\scriptsize
\begin{tabular}{p{0.20\textwidth}p{0.13\textwidth}p{0.09\textwidth}p{0.12\textwidth}p{0.40\textwidth}}
\toprule
Item & Symbol & Value & Unit & Notes \\
\midrule
Batch size & $B$ & 64 & samples & One batch. \\
Sequence length & $S$ & 128 & tokens & Fixed for every task, including SST-2. \\
Hidden width & $C$ & 768 & channels & \\
FFN width & $F$ & 3072 & channels & \\
Attention heads & $h$ & 12 & heads & \\
Head width & $d_k$ & 64 & channels & $C/h$. \\
Time steps & $T$ & 16 & steps & Per-step rate is $s/T$. \\
Encoder blocks & $L$ & 12 & blocks & Amortizes the QNN pooler transfer. \\
Eight-position fanout & $N_8$ & $5.96\times10^{10}$ & deliveries/block & $4BSC^2 + 2BhS^2d_k + 2BSCF$. \\
\bottomrule
\end{tabular}
\endgroup
\end{table*}

\begin{table*}[t]
\centering
\caption{Energy breakdown per encoder block on SST-2, in
mJ/block. Values are shown to four significant digits; component
sums equal the totals exactly in the unrounded data.}
\label{tab:appendix_c2_breakdown}
\begingroup
\setlength{\tabcolsep}{7pt}
\small
\begin{tabular}{lrrrr}
\toprule
Model & Move & W-read & Compute & Total \\
\midrule
SpikingBERT & 715.1 & 25.38 & 10.36 & 750.8 \\
SpikingLM & 236.0 & 7.956 & 3.408 & 247.4 \\
Sorbet & 371.9 & 13.20 & 5.419 & 390.5 \\
Otters & 137.8 & 5.378 & 3.858 & 147.1 \\
\midrule
Matterhorn QNN (dense) & 61.20 & 6.348 & 3.991 & 71.54 \\
Bit-serial QNN ($k{=}0$ pattern) & 192.5 & 6.318 & 2.763 & 201.6 \\
Bit-serial QNN ($k{=}1$ pattern) & 114.1 & 3.743 & 1.640 & 119.4 \\
Bit-serial QNN ($k{=}2$ pattern) & 77.38 & 2.539 & 1.115 & 81.03 \\
Matterhorn TTFS & 110.0 & 4.424 & 3.096 & 117.5 \\
Matterhorn M-TTFS ($k{=}0$) & 72.07 & 3.077 & 2.051 & 77.19 \\
Matterhorn M-TTFS ($k{=}1$) & 44.86 & 2.112 & 1.301 & 48.27 \\
\bottomrule
\end{tabular}
\endgroup
\end{table*}

\begin{figure*}[t]
\centering
\includegraphics[width=0.76\textwidth]{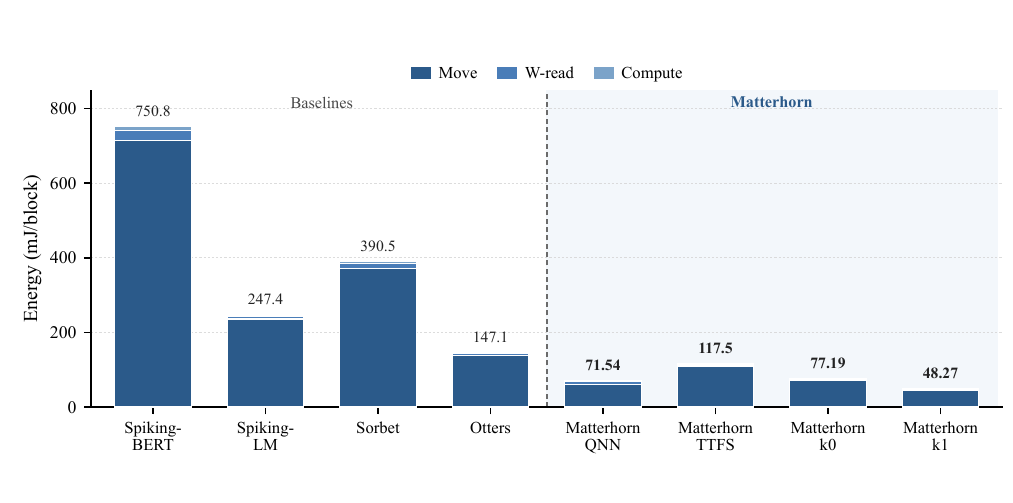}
\caption{Stacked component breakdown of
Table~\ref{tab:appendix_c2_breakdown}. Movement dominates every
configuration, which is why the encoding, rather than the
arithmetic, decides the ranking.}
\label{fig:appendix_c3}
\end{figure*}

\end{document}